\documentclass[letterpaper]{article}
\usepackage{aaai2027}
\usepackage[hyphens]{url}
\usepackage{graphicx}
\usepackage{natbib}
\usepackage{caption}
\usepackage{algorithm}
\usepackage{algorithmic}
\usepackage{booktabs}
\usepackage{amsmath}
\usepackage{amssymb}
\usepackage{array}

\newcommand{\carve}{\textsc{CARVE}}
\newcommand{\method}{\textsc{CARVE-Q}}

\newcommand{\rules}{\mathcal{H}}
\newcommand{\cost}{\Phi}
\newcommand{\cert}{\mathcal{C}}
\newcommand{\fallback}{\mathcal{A}_{\mathrm{fb}}}
\newcommand{\bbone}{\mathbf{1}}
\newtheorem{theorem}{Theorem}
\newtheorem{lemma}{Lemma}
\newtheorem{definition}{Definition}

\title{CARVE-Q: Quantum-Proposed, Classically Certified Interactive Driving Repair}
\author{
Yifan Wang
}
\affiliations{
Department of Mechanical Engineering, McGill University\\
QC, H3A 2T7, Canada\\
\texttt{yifan.wang18@mail.mcgill.ca}
}

\begin{document}
\maketitle

\begin{abstract}
The critical question after a correct driving veto is not only whether a
maneuver is unsafe, but whether the blocked interaction admits a lawful,
auditable, and responsibility-bounded repair. Prediction and game-theoretic
planners can suggest plausible cooperation, yet they do not return a proof that
the repair respects hard rules, right-of-way, cost allocation, and ego fallback.
We introduce \carve, \emph{Certified Affordable Repair of Vetoed maneuvers via
Envelopes}, a certificate architecture for prediction-free interactive repair.
Given a vetoed maneuver, \carve\ constructs a finite repair lattice and emits a
structured certificate recording the binding rule, selected joint repair,
right-of-way-scaled cooperation envelope, responsibility-weighted cost split, and
ego-only fallback. This certificate view reveals the algorithmic bottleneck:
multi-owner repair induces a product lattice
\(M=\prod_j|\mathcal A_j|\). We therefore introduce \method, a verifier-shielded
quantum-AI search layer that applies quantum minimum finding only to this
black-box lattice while leaving all safety authority classical. In the
conservative verifier-oracle model, exact classical minimum finding requires
\(\Theta(M)\) queries in the worst case, whereas Durr--Hoyer/Grover minimum
finding uses \(O(\sqrt M)\) oracle queries with high probability. We prove
verifier-shielded certificate soundness, priority non-elicitation, black-box
query separation, and finite-precision reversible-oracle constructibility. We
then demonstrate statevector minimum finding on \carve\ repair oracles up to
65,536 assignments and validate certificate preservation on Lanelet2-grounded
INTERACTION replay with 100\% right-of-way respect, 100\% blame consistency, and
zero priority false positives. The result is a trust-bounded quantum-AI pattern
for certified autonomy: quantum proposes; \carve\ certifies.
\end{abstract}

\section{Introduction}

The most important decision in interactive driving may begin after the correct
veto. Consider an ego vehicle entering a dense urban merge. A hard-rule gate
rejects the maneuver because another vehicle is marginally inside the required
gap. The rejection is safe, but it is not yet an explanation of what should
happen next. A predictor may guess that the other driver will slow down; a
game-theoretic planner may search for a compatible response. Neither object is
the proof a safety-critical system needs: a certificate stating why the maneuver
was vetoed, what bounded repair is allowed, who is asked to accommodate, how the
cost is assigned, and what ego can still do safely if cooperation does not
arrive.

This example exposes a disconnect. Hard-rule vetoes are essential for
unrecoverable danger, yet they are behaviorally rigid when a conflict can be
resolved through a small, bounded accommodation such as waiting, yielding, or
decelerating. Prediction and social driving models estimate likely interactions
\citep{sadigh2016planning,schwarting2019social,alahi2016social,
deo2018convolutional,salzmann2020trajectron}, while rulebooks, RSS, shielding,
and formal-safety methods specify constraints \citep{shalev2017rss,
censi2019rulebooks,alshiekh2018shielding,garcia2015safe}. The missing object is
neither another trajectory predictor nor another veto rule; it is a
certificate-bearing interactive repair.

We introduce \carve, \emph{Certified Affordable Repair of Vetoed maneuvers via
Envelopes}. \carve\ elevates a rejected maneuver from terminal failure into a
finite repair problem whose output is an auditable liability record. The
certificate names the binding hard rule, selects a bounded joint repair, checks
right-of-way-scaled cooperation envelopes, records a responsibility-weighted
cost split, and preserves an ego-only fallback. A repairable interaction must
therefore answer three questions at once: does the repair make every declared
hard-rule margin nonnegative, does every requested accommodation stay inside a
normatively admissible envelope, and can ego still recover without assuming
another driver's compliance?

\carve\ also exposes a search bottleneck. If \(n\) repair owners each have a
finite action set \(\mathcal A_j\), exact joint repair searches a product
lattice of size \(M=\prod_j|\mathcal A_j|\). Dense unprotected turns, narrow
construction-zone negotiations, and multi-agent merge conflicts can require
several vehicles to choose among wait, decelerate, yield, or no-op edits. The
quantum component is not imposed on the problem; it appears only after \carve\
converts interaction repair into this finite certificate lattice.

The black-box verifier model is the safety assumption that makes the quantum
layer meaningful. If the verifier exposes stable convex, separable, or
low-treewidth structure, white-box classical solvers should exploit it. In
safety-critical integration, however, rulebooks can change, predicates can be
nonconvex or adversarially coupled, and future shields may include proprietary
perception modules. A trust-bounded design therefore assumes no exploitable
structure and analyzes the verifier as an opaque cost oracle. Under this model,
exact classical minimum finding has a worst-case \(\Theta(M)\) query
requirement, whereas Grover/Durr--Hoyer minimum finding gives \(O(\sqrt M)\)
oracle queries with high probability \citep{grover1996,durr1996minimum,
bennett1997strengths,boyer1998tight,brassard2002amplitude}. This is the
precise sense in which quantum search is necessary in \method.

\method\ follows one trust boundary throughout: quantum proposes; \carve\
certifies. The quantum module searches for a low-cost joint repair, but the
deterministic \carve\ verifier recomputes every predicate before any
certificate can be emitted. We do not claim a new quantum search algorithm,
present-day hardware speedup, universal superiority over white-box solvers, or
quantum-certified safety. The claim is architectural and oracle-theoretic:
\carve\ defines the certified repair object, \method\ targets its black-box
joint-lattice bottleneck, and \carve\ remains the final certificate authority.

Our contributions are:
\begin{itemize}
\item \textbf{Certified interactive repair.} We formulate hard-rule-vetoed
driving interaction as a finite repair-certificate problem rather than a
prediction, game response, or ego-only replanning problem.
\item \textbf{\carve\ certificate semantics.} We define right-of-way-scaled
cooperation envelopes \(B_j(s)=\beta(\pi_j)\alpha_j^{\max}(s)\),
responsibility-weighted costs, affordability predicates, and fallback validity.
\item \textbf{Verifier-shielded quantum-AI search.} We show that multi-owner
repair induces a product lattice \(M=\prod_j|\mathcal A_j|\) and introduce
\method, which applies quantum minimum finding to the black-box verifier-cost
oracle while preserving classical certificate authority.
\item \textbf{Guarantees and evidence.} We prove verifier-shielded soundness,
priority non-elicitation, black-box query separation, and reversible-oracle
constructibility; we validate them with statevector minimum finding, black-box
stress, resource accounting, replay diagnostics, and QAOA/local-search audits.
\end{itemize}

\begin{figure*}[t]
\centering
\includegraphics[width=\textwidth]{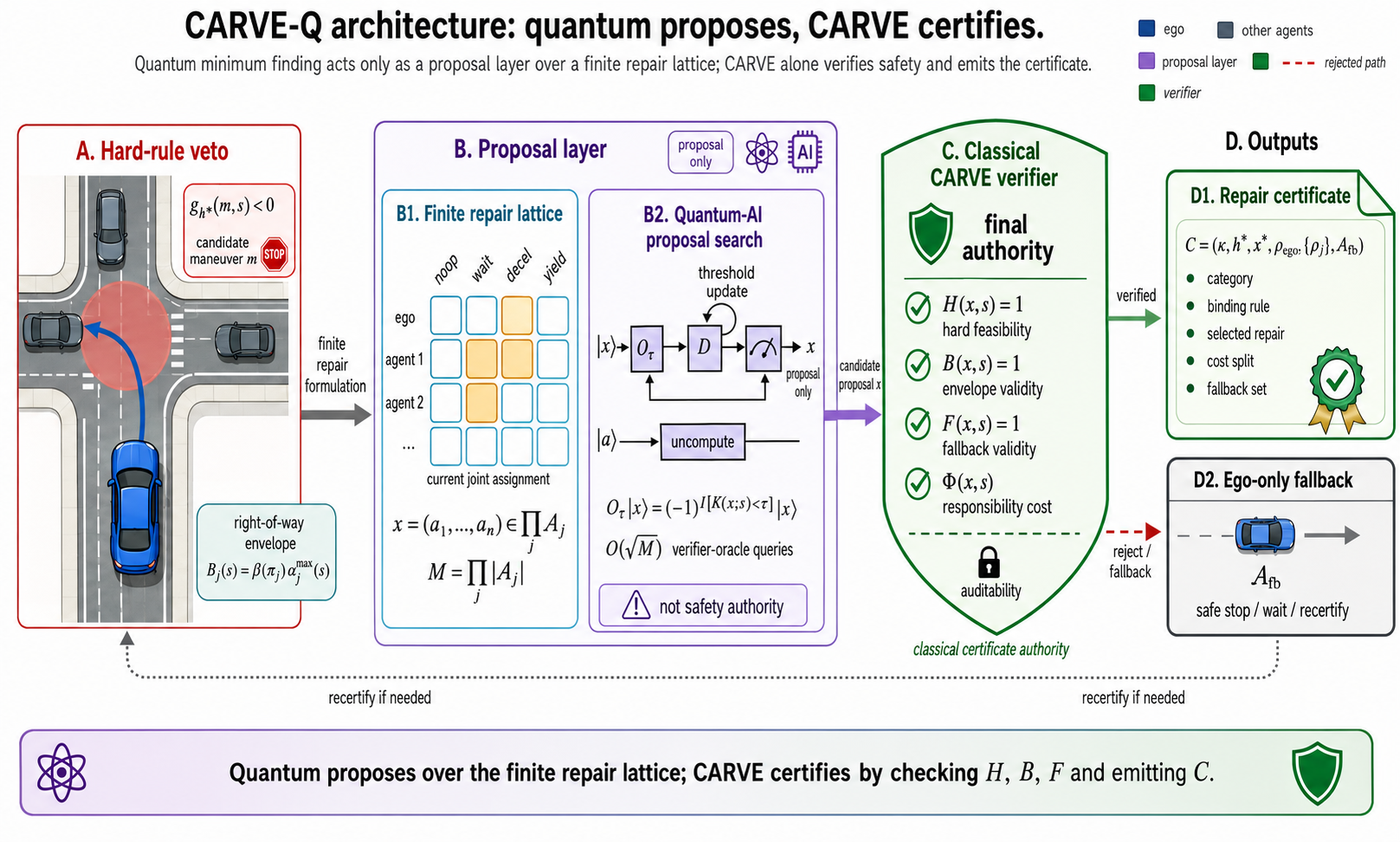}
\caption{\method\ architecture. \carve\ is the certified repair architecture:
it elevates a hard-rule veto into an auditable certificate, exposes a finite
multi-owner repair lattice, and lets quantum minimum finding search only that
black-box bottleneck. The quantum module is a proposal layer only; the final
certificate is emitted by the classical \carve\ verifier.}
\label{fig:architecture}
\end{figure*}

\section{From Vetoed Maneuvers to CARVE Certificates}

\carve\ begins from a different object than prediction-based interaction
planning: a repair certificate. The input is a scene \(s\), a candidate maneuver
\(m\), and a hard-rule prefix \(\rules\). If all margins are nonnegative,
\carve\ returns an empty satisfied certificate. If a binding rule \(h^\star\)
has negative margin, \carve\ constructs a finite repair lattice with ego-owned
edits and bounded agent-owned accommodation requests.

\begin{definition}[\carve\ certificate]
Given a scene \(s\), a vetoed maneuver \(m\), and a binding hard rule
\(h^\star\), a \carve\ certificate is
\[
\cert=(\kappa,h^\star,x^\star,\rho_{\mathrm{ego}},\{\rho_j\}_{j=1}^{n},\fallback),
\]
where \(\kappa\) is the certificate category, \(x^\star\) is the selected joint
repair assignment, \(\rho_{\mathrm{ego}}\) and \(\rho_j\) are
responsibility-weighted cost allocations, and \(\fallback\) is an ego-only
fallback action set executable without external cooperation.
\end{definition}

\begin{figure}[t]
\centering
\includegraphics[width=\columnwidth]{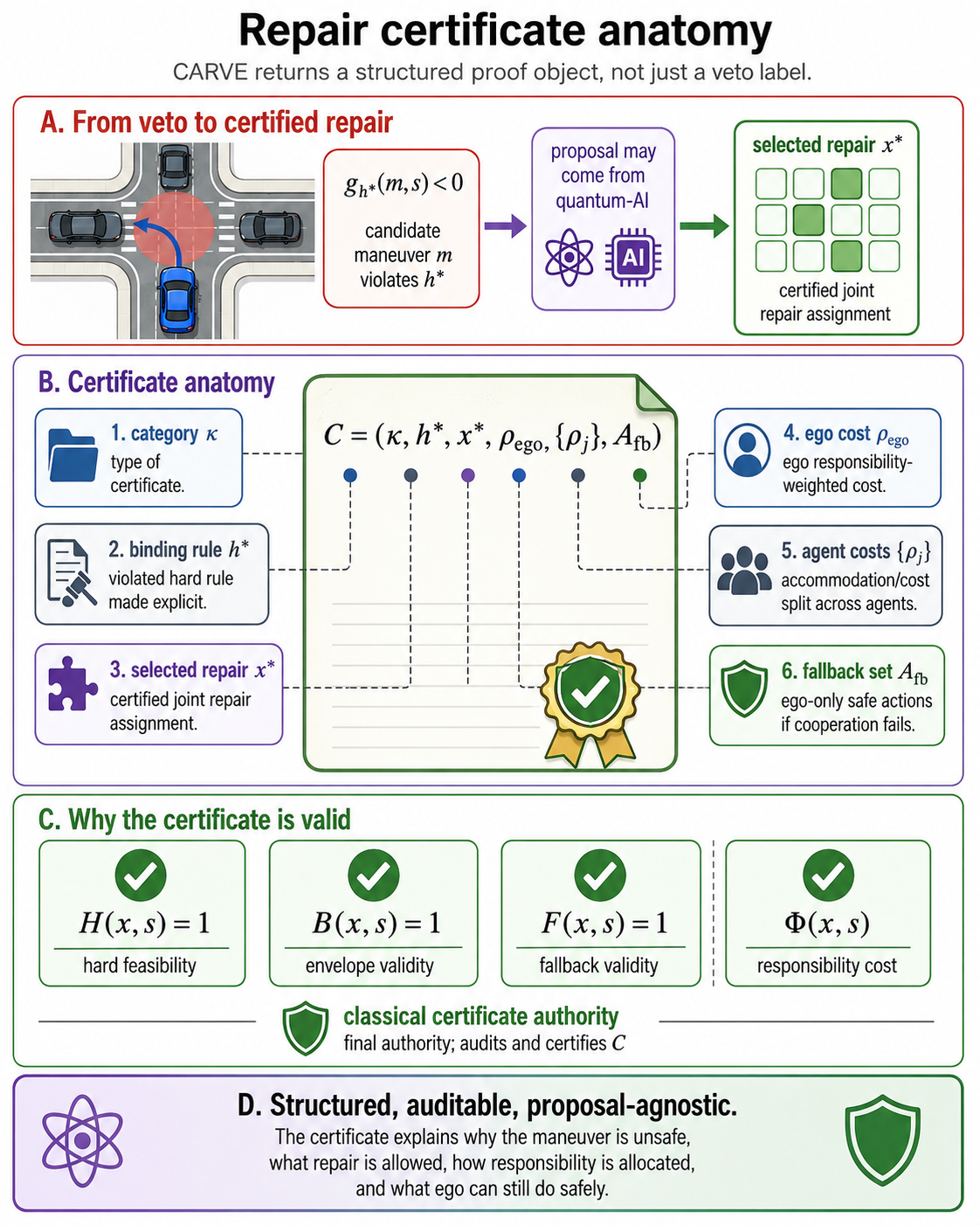}
\caption{Repair certificate anatomy. \carve\ returns a structured proof object,
not merely a veto label or a predicted trajectory. A proposal may come from
quantum search, classical search, or another heuristic; acceptance depends only
on the classical verifier recomputing \(H\), \(B\), \(F\), and \(\Phi\).}
\label{fig:certificate}
\end{figure}

The design separates three questions that are often conflated. First, hard-rule
feasibility asks whether applying a repair makes all declared safety margins
nonnegative. Second, affordability asks whether ego effort and every requested
agent accommodation remain within declared budgets. Third, fallback asks
whether ego retains an executable contingency if requested cooperation is not
observed. The certificate is accepted only when these predicates pass.

This distinction is what makes \method\ an AI architecture rather than only a
quantum algorithm. The quantum subroutine is useful because \carve\ defines a
finite but exponentially growing decision object. The verifier shield is useful
because safety-critical semantics remain interpretable and classically
auditable.

\section{Certified Interactive Repair Semantics}

\paragraph{Rules and owners.}
Let \(s\) contain ego state, agent states, semantic map context, and priority
roles \(\pi_j\). A hard rule \(h_\ell\in\rules\) returns margin
\(g_\ell(m,s)\). A repair owner can be ego or an interacting agent. Owner \(j\)
has a finite set \(\mathcal A_j\) containing no-op and tactical edits such as
wait, decelerate, yield, or nudge. A joint repair assignment is
\[
x=(a_1,\ldots,a_n)\in\mathcal A_1\times\cdots\times\mathcal A_n .
\]
The lattice uses semantic repair prototypes rather than arbitrary real-valued
controls. Continuous quantities such as delay or deceleration are discretized
conservatively into resolution cells. If rule margins are Lipschitz in these
repair parameters, refinement bounds the margin degradation between a feasible
continuous repair and its representative prototype; missed repairs are
therefore controlled false negatives, not unsafe false positives, because every
accepted assignment is rechecked by the classical verifier.

\paragraph{Right-of-way cooperation envelopes.}
\carve\ does not authorize arbitrary cooperation requests. For an interacting
agent \(j\), the request magnitude \(\Delta_j\) must satisfy
\[
0\le \Delta_j \le B_j(s),\qquad
B_j(s)=\beta(\pi_j)\alpha_j^{\max}(s).
\]
The term \(\alpha_j^{\max}(s)\) is a conservative kinematic accommodation
bound, such as the maximum safe yield or speed reduction available to agent
\(j\) under current speed and road conditions. The factor
\(\beta(\pi_j)\in[0,1]\) scales that bound by semantic right-of-way. Priority
holders use \(\beta(\pi_j)=0\), so no nonzero request to a priority agent can be
certified. This structurally separates physical reachability from normative
admissibility.

\paragraph{Cost and fallback.}
The objective \(\cost(x,s)\) is a responsibility-weighted cost over selected
ego and agent edits:
\[
\cost(x,s)=\rho_{\mathrm{ego}}(x,s)+\sum_j w(\pi_j)\rho_j(x,s).
\]
Here \(\rho_j(x,s)\ge 0\) is a kinematic or temporal penalty, such as delay or
integrated deceleration, and \(w(\pi_j)>0\) is a normative multiplier that
penalizes inappropriate burden shifting. Fallback prevents the system from
accepting a maneuver that consumes the last recovery option. In an elicited or joint certificate,
\(\fallback\) is an ego-only contingency, such as a safe stop or wait action,
that remains executable without relying on another driver's compliance. If the
requested accommodation is not observed, ego follows the fallback or recertifies
before proceeding. The certificate never asserts that another driver will
comply.

\paragraph{Certified joint-repair problem.}
Let \(H(x,s)\), \(B(x,s)\), and \(F(x,s)\) denote hard feasibility,
affordability/envelope validity, and fallback validity. The finite repair
problem is
\begin{align*}
x^\star &\in \arg\min_{x\in\mathcal X(s)} \cost(x,s)\\
\text{s.t.}\quad H(x,s)&=1,\quad B(x,s)=1,\quad F(x,s)=1,\\
0&\le \Delta_j(x,s)\le B_j(s),\quad \forall j,
\end{align*}
where \(\mathcal X(s)=\prod_j\mathcal A_j(s)\) and
\(M(s)=|\mathcal X(s)|\).
\carve\ can solve this exactly over the finite lattice, greedily for online
use, or through a verifier-oracle search layer. Only the verifier can emit
\(\cert\).

\begin{table}[t]
\centering
\setlength{\tabcolsep}{2pt}
\begin{tabular}{@{}lccccc@{}}
\toprule
Method & Cert. & RoW & Resp. & Fall. & Multi \\
\midrule
Hard gate & -- & \(\checkmark\) & -- & -- & -- \\
Prediction & -- & -- & -- & -- & \(\checkmark\) \\
Ego repair & partial & -- & -- & \(\checkmark\) & -- \\
Exact & \(\checkmark\) & \(\checkmark\) & \(\checkmark\) & \(\checkmark\) & \(\checkmark\) \\
\method & \(\checkmark\) & \(\checkmark\) & \(\checkmark\) & \(\checkmark\) & \(\checkmark\) \\
\bottomrule
\end{tabular}
\caption{CARVE-side capability gap. Exact denotes exhaustive enumeration over
the joint lattice; \method\ adds verifier-shielded quantum search while
retaining the same certificate properties.}
\label{tab:capability}
\end{table}

\section{Verifier-Shielded Quantum-AI Search}

Given the \carve\ verifier, define a finite-precision verifier-cost oracle.
Let \(\Phi_{\max}\) be a saturation value larger than any feasible encoded
repair cost. Then
\[
\tilde f(x;s)=
\begin{cases}
\cost(x,s), & H(x,s)\wedge B(x,s)\wedge F(x,s)=1,\\
\Phi_{\max}, & \text{otherwise}.
\end{cases}
\]
Infeasible assignments therefore implement \(+\infty\)-style semantics in a
bounded arithmetic register. To remove degeneracy, define the composite key
\[
K(x;s)=(\tilde f(x;s),\operatorname{lex}(x)),
\]
where \(\operatorname{lex}(x)\) is the integer assignment encoding. For
threshold key \(\tau\), the phase oracle marks lower-cost feasible assignments:
\[
O_\tau |x\rangle=(-1)^{\bbone[K(x;s)<_{\mathrm{lex}}\tau]}|x\rangle .
\]
Here \(\bbone[\cdot]\) is the 0/1 indicator. One verifier-oracle query is one
application of this reversible threshold phase oracle, including predicate
evaluation, cost computation, comparison, phase flip, and uncomputation.
\method\ runs a Durr--Hoyer-style minimum-finding loop over this oracle
\citep{durr1996minimum}: initialize a feasible threshold, amplify assignments
below the threshold, sample a candidate, update the threshold if the candidate
is lower cost, and finally pass the best candidate to the \carve\ verifier.
Minimum finding is the right primitive because certified repair is a finite
verifier-oracle optimization problem, not a policy-learning problem. Variational
quantum methods and QAOA are useful context, but they do not provide the
load-bearing worst-case guarantee used here.

\begin{algorithm}[t]
\caption{\method\ Verifier-Shielded Repair Search}
\label{alg:carveq}
\begin{algorithmic}[1]
\STATE Build the finite \carve\ repair lattice.
\STATE Initialize threshold \(\tau\) from any feasible repair or fallback.
\IF{no finite feasible threshold is available}
  \STATE \textbf{return} refusal certificate or ego fallback.
\ENDIF
\REPEAT
  \STATE Reversibly construct \(O_\tau\), mapping infeasible states to \(\Phi_{\max}\).
  \STATE Apply amplitude amplification over assignment states.
  \STATE Sample candidate \(\hat{x}\); evaluate \(K(\hat{x};s)\) classically.
  \IF{\(K(\hat{x};s)<_{\mathrm{lex}}\tau\)}
    \STATE Store \(\hat{x}\) and set \(\tau\leftarrow K(\hat{x};s)\).
  \ENDIF
\UNTIL{threshold stops improving or the prescribed minimum-finding schedule ends}
\STATE Re-run the \carve\ verifier on the stored candidate.
\STATE Emit \(\cert\) only if \(H=1\), \(B=1\), and \(F=1\); otherwise return fallback/refusal.
\end{algorithmic}
\end{algorithm}

\section{Theory and Oracle Construction}

\begin{theorem}[Verifier-shielded certificate soundness]
For any proposal generator \(G\), including quantum search, classical search,
random sampling, or a faulty heuristic, if \method\ returns an accepting
certificate, then the selected repair satisfies hard-rule feasibility,
right-of-way envelope validity, affordability, responsibility accounting, and
fallback validity.
\end{theorem}
\noindent\emph{Proof sketch.}
The quantum routine is not a certificate authority. The final assignment is
accepted only after the deterministic \carve\ verifier re-runs the entire
predicate suite on the returned \(\hat{x}\). The certificate is emitted only if
hard feasibility, envelope validity, affordability, and fallback validity all
pass. Thus soundness is independent of whether the proposal came from quantum
search, classical search, random sampling, or a faulty heuristic. Full proofs
are in the supplementary material.

\begin{lemma}[Priority non-elicitation]
If \(\beta(\pi_j)=0\) for a priority holder \(j\), then no accepting
\carve\ certificate contains a positive requested accommodation from \(j\).
\end{lemma}
\noindent\emph{Proof sketch.}
Priority implies \(B_j(s)=0\). Any positive request has \(\Delta_j>0\), so it
violates \(0\le\Delta_j\le B_j(s)\) and is rejected by the envelope predicate.

\begin{theorem}[Verifier-oracle query separation]
In the black-box verifier-cost oracle model with a finite lattice of size
\(M\), finite-precision costs, lexicographic tie-breaking by \(K(x;s)\), and
oracle access only, any deterministic classical algorithm that exactly solves
minimum finding requires \(\Theta(M)\) queries in the worst case; bounded-error
randomized algorithms require \(\Omega(M)\). Quantum minimum finding returns the
distinguished minimum-cost feasible assignment with high probability using
\(O(\sqrt{M})\) oracle queries.
\end{theorem}
\noindent This is a black-box verifier-oracle result. It does not assert
dominance over white-box classical solvers such as branch-and-bound, CP-SAT,
MILP, or structure-exploiting local search on structured instances; E2 isolates
the black-box regime empirically. If no feasible assignment exists,
\(\tilde f(x;s)=\Phi_{\max}\) for all \(x\), and the verifier issues a refusal
or fallback certificate.
\noindent\emph{Proof sketch.}
The classical lower bound follows from the unstructured search/minimum-finding
lower bound: a black-box table can hide its unique optimum at any unqueried
location \citep{bennett1997strengths,boyer1998tight}. Lexicographic
tie-breaking converts cost-degenerate optima into a distinguished minimum. The
quantum upper bound follows from amplitude amplification and minimum finding.

For scale, a 10-owner lattice with four choices per owner has
\(M=4^{10}=1{,}048{,}576\) and \(\sqrt{M}=1{,}024\), before accounting for
constant factors and oracle-construction cost.

\begin{theorem}[Polynomial reversible constructibility]
With finite action encodings, \(p\) predicate blocks, pairwise constraints, and
fixed-point bit-width \(b\) for cost and request magnitudes, the \carve\
verifier-cost threshold predicate can be embedded into a reversible phase oracle
with \(\operatorname{poly}(n,p,b)\) logical gates and ancillas.
\end{theorem}
\noindent\emph{Proof sketch.}
The verifier is composed of finite equality decoders, controlled additions,
comparators, Boolean conjunctions, a phase flip, and uncomputation. Each block
has polynomial size in the encoded scene description. The exponential term is
the number of assignments \(M\), not the cost of one oracle evaluation.
Conservative finite-precision bounds can introduce false negatives that waste
search effort, but not unsafe false positives, because the continuous
classical verifier rechecks every returned assignment before certification.

\section{Experiments}

The experiments follow the paper's trust boundary. E1 demonstrates minimum
finding on \carve\ verifier-cost oracles. E2 isolates the black-box regime by
destroying semantic neighborhood structure. E3 checks reversible-oracle
constructibility. E4 validates certificate semantics on replay data. E5 reports
a QAOA/local-search audit explaining why minimum finding, not variational
optimization, is the load-bearing quantum primitive.

\paragraph{E1: Statevector minimum finding.}
We execute phase-oracle flips and diffusion steps on \carve\ repair oracles.
Because the simulated cost table is known, the statevector experiment uses the
number of marked states to select near-optimal Grover rotation counts at each
threshold. This is a noise-free oracle-model demonstration; standard
Durr--Hoyer schedules with unknown marked count use randomized rotation
schedules and retain the same \(O(\sqrt M)\) expected query order. Calibration
only improves constants for transparent query accounting.

\begin{table}[t]
\centering
\setlength{\tabcolsep}{3pt}
\begin{tabular}{@{}rrrrr@{}}
\toprule
Agents & \(M\) & \(\sqrt M\) & DH calls & Success \\
\midrule
2 & 16 & 4 & \(2.63\pm1.80\) & 1.00 \\
3 & 64 & 8 & \(8.75\pm2.59\) & 1.00 \\
4 & 256 & 16 & \(12.00\pm8.92\) & 1.00 \\
5 & 1,024 & 32 & \(32.00\pm20.39\) & 1.00 \\
6 & 4,096 & 64 & \(72.88\pm57.79\) & 1.00 \\
7 & 16,384 & 128 & \(180.63\pm60.52\) & 1.00 \\
8 & 65,536 & 256 & \(434.38\pm161.79\) & 1.00 \\
\bottomrule
\end{tabular}
\caption{Statevector minimum-finding query scaling on \carve\ repair oracles.
DH denotes Durr--Hoyer-style measured phase-oracle calls; values are
mean\(\pm\)std over eight instances per size.}
\label{tab:statevector}
\end{table}

At eight agents with four choices each, exact enumeration evaluates 65,536
assignments; the calibrated statevector routine uses 434.38 phase-oracle calls
on average and all returned assignments pass classical verification.

\paragraph{E2: Black-box relabeling stress.}
To isolate the theorem's model, we apply a secret reversible relabeling to the
same \carve\ verifier-cost table. This preserves the table but destroys
semantic neighborhood structure. In a white-box lattice, a small deceleration
edit may be adjacent to a medium deceleration edit; after relabeling, Hamming
neighbors no longer track semantic similarity. Same-budget bit-flip local search
and random search then lose their neighborhood advantage, while minimum finding
remains exact.

\begin{table}[t]
\centering
\setlength{\tabcolsep}{4pt}
\begin{tabular}{rrrrr}
\toprule
Agents & \(M\) & DH & Local & Random \\
\midrule
4 & 256 & 1.00 & 0.38 & 0.38 \\
5 & 1,024 & 1.00 & 0.00 & 0.00 \\
6 & 4,096 & 1.00 & 0.00 & 0.00 \\
7 & 16,384 & 1.00 & 0.00 & 0.00 \\
8 & 65,536 & 1.00 & 0.00 & 0.00 \\
\bottomrule
\end{tabular}
\caption{Black-box relabeled oracle stress. Entries are exact-hit rates under
matched oracle-interaction budgets. The verifier-cost table is unchanged; only
assignment labels are reversibly relabeled to remove semantic locality.}
\label{tab:blackbox}
\end{table}

\paragraph{E3: Oracle construction counts.}
The largest counted setting has 10 agents, 20 assignment bits, 45 pair
constraints, 22,081 Toffoli gates, 51,872 CNOT gates, and 1,104 logical
qubits. These counts support constructibility and reveal overhead; they do not
claim current hardware wall-clock advantage. The construction targets a
fault-tolerant quantum-computing setting; wall-clock advantage depends on future
hardware and error-correction overhead.

\paragraph{E4: Lanelet2-grounded replay.}
We use 589 INTERACTION replay episodes with Lanelet2 geometry
\citep{zhan2019interaction,poggenhans2018lanelet2}. Replay Human Alignment
(RHA) is a behavioral diagnostic, not a safety metric. False-veto recovery rate
(FVRR) measures accepted certificates over initially hard-vetoed but
human-resolved episodes. Blame-consistency rate (BCR) measures whether the
responsibility allocation respects distinct-duty ordering; pairwise BCR is its
pair-level form. The RHA improvement comes from the \carve\ certificate layer,
not from the quantum backend: \method\ searches for assignments that must
satisfy the same predicates. The verifier-gated layer improves full replay RHA from
28.23\% to 41.82\% and held-out RHA from 26.97\% to 40.45\%, while preserving
97.88\% FVRR, 100\% right-of-way (RoW) respect, 100\% BCR, 100\% pairwise BCR,
and zero priority false positives. The RHA gain is not imitation accuracy:
rule-compliant certificates should diverge from aggressive or unlawful human
behavior. It indicates better recovery of lawful compromises without weakening
certificate predicates.

\begin{table}[t]
\centering
\setlength{\tabcolsep}{3pt}
\begin{tabular}{lrrl}
\toprule
Metric & Base & Gated & Role \\
\midrule
Full RHA & 28.23 & 41.82 & diagnostic \\
Held-out RHA & 26.97 & 40.45 & diagnostic \\
FVRR & 97.88 & 97.88 & certificate \\
RoW respect & 100.00 & 100.00 & certificate \\
BCR & 100.00 & 100.00 & certificate \\
Priority FP & 0 & 0 & certificate \\
\bottomrule
\end{tabular}
\caption{Replay metrics. RHA is diagnostic only; FVRR, right-of-way respect,
BCR, and priority false positives are certificate metrics.}
\label{tab:replay}
\end{table}

\begin{figure*}[t]
\centering
\includegraphics[width=\textwidth]{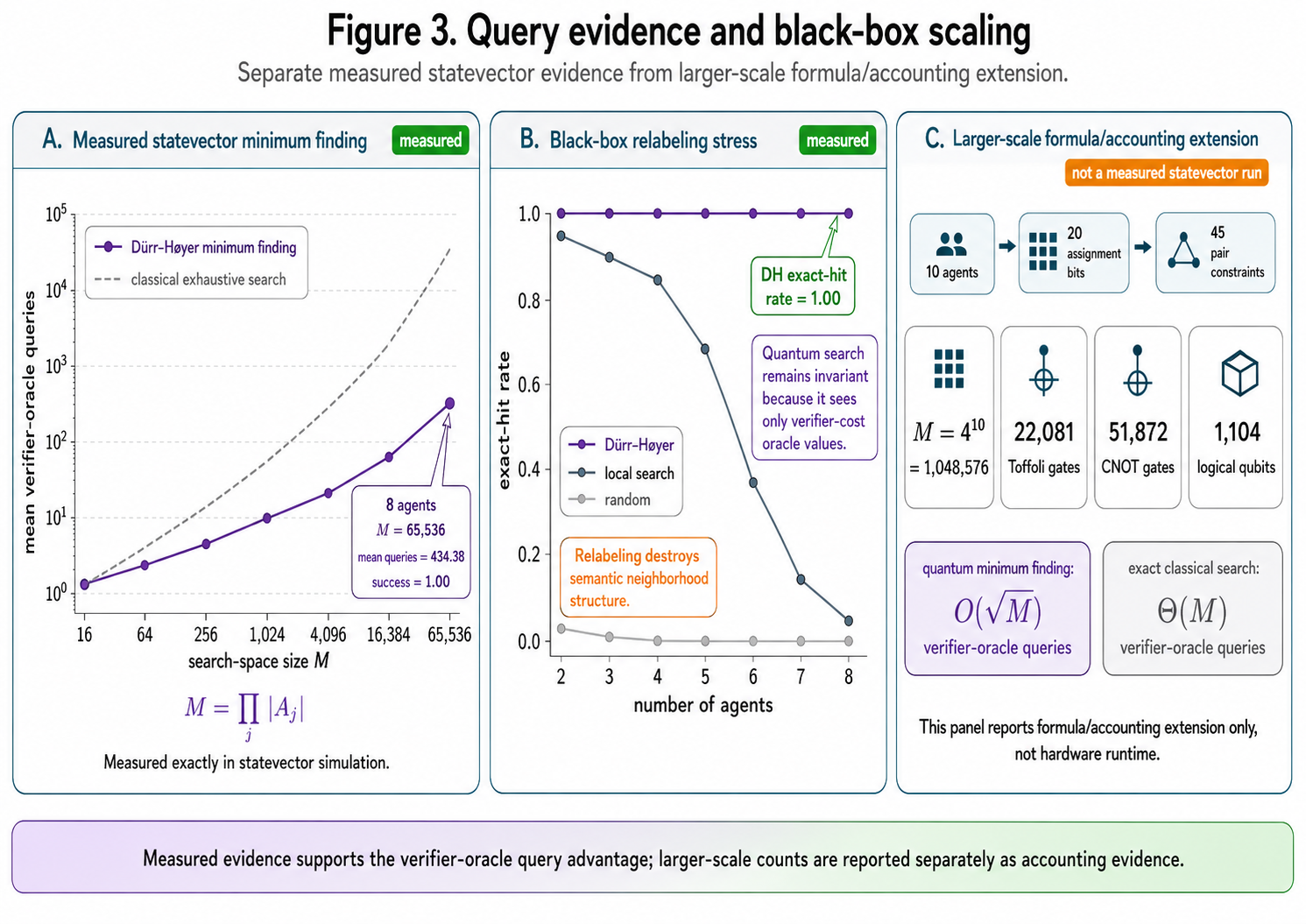}
\caption{Query evidence and black-box scaling. Panels A and B report measured
statevector evidence on \carve\ verifier-cost oracles; Panel C is explicitly a
larger-scale formula/accounting extension, not a measured statevector or
hardware-runtime result.}
\label{fig:evidence}
\end{figure*}

\paragraph{E5: Design-discipline audit.}
Small white-box QUBO tasks favor multi-start local search: local search hits
exact optima in 96.67\% of tasks, while \(p=1\) QAOA hits 23.33\%. This is not
a weakness of the paper's main claim; it is why \method\ uses theorem-backed
minimum finding rather than making QAOA the load-bearing result. Local search
dominates only when it can exploit the white-box neighborhood structure that
the conservative black-box model refuses to assume.

\begin{table}[t]
\centering
\setlength{\tabcolsep}{3pt}
\begin{tabular}{@{}p{0.42\columnwidth}p{0.42\columnwidth}@{}}
\toprule
\textbf{Supported} & \textbf{Excluded} \\
\midrule
Certified repair semantics & Hardware speedup \\
Verifier-shielded soundness & Quantum safety certification \\
Black-box query separation & Universal quantum superiority \\
Statevector oracle evidence & QAOA dominance \\
Replay certificate preservation & Deployment readiness \\
\bottomrule
\end{tabular}
\caption{Claim boundary. The paper's evidence supports a verifier-oracle
quantum-AI search role inside \carve; every safety-critical claim remains
classically certified.}
\label{tab:boundary}
\end{table}

\section{Related Work}

\textbf{Rule-constrained and certified autonomy.}
RSS, rulebooks, runtime shielding, and formal safety define interpretable
constraints and interventions \citep{shalev2017rss,censi2019rulebooks,
alshiekh2018shielding,garcia2015safe,kochenderfer2012acas,fraichard2007ics}.
They sit within a broader motion-planning tradition that includes sampling,
optimal planning, and urban-driving surveys \citep{lavalle2006planning,
lavalle1998rrt,karaman2011sampling,paden2016survey}.
\method\ differs by asking whether a vetoed maneuver has an attributable,
affordable, multi-owner repair certificate.
Ego-only recovery and trajectory repair alter a geometric plan; \carve\ repairs
an interactive decision object and records who owns each bounded accommodation.

\textbf{Interactive prediction and game-theoretic driving.}
Prediction and interaction-aware planning model how agents may respond
\citep{sadigh2016planning,schwarting2019social,kuderer2015learning,
alahi2016social,deo2018convolutional,salzmann2020trajectron}. Prediction can
be useful for candidate generation, but likelihood is not a certificate.
\carve\ uses replay alignment as a diagnostic while keeping certification in
rule predicates.

\textbf{Multi-agent combinatorial search.}
MAPF and coordination methods study joint decisions in difficult search spaces
\citep{silver2005cooperative,sharon2015cbs,stern2019mapf,ma2017lifelong,
okumura2023lacam,van2011reciprocal}. \carve\ differs because each assignment
must also satisfy driving-specific hard rules, right-of-way envelopes,
responsibility costs, and fallback validity.

\textbf{Quantum AI and quantum optimization.}
Quantum machine learning explores feature maps, variational policies, and QAOA
\citep{biamonte2017qml,schuld2019quantum,havlicek2019supervised,
broughton2020tensorflow,jerbi2021variational,farhi2014qaoa,preskill2018nisq,
cerezo2021variational,mcclean2018barren}; standard quantum computation
references provide the circuit and oracle model used by the minimum-finding
layer \citep{nielsen2010quantum}. \method\ takes a different
quantum-AI route: it embeds a provable quantum search primitive inside a
symbolic safety architecture and leaves the certificate auditable and
classically verified.
\method\ does not claim a new quantum search algorithm; the contribution is the
reduction of certified repair to a verifier-oracle minimum-finding problem and
the verifier-shielded safety architecture around that reduction.

\section{Discussion and Claim Boundary}

\method\ contributes to trustworthy quantum-AI integration: it does not replace
classical AI safety with a quantum device, but uses quantum computation to
search a formally defined bottleneck. This is why the negative QAOA result
helps the story. Variational heuristics may be useful in other settings, but
the load-bearing result here is verifier-oracle minimum finding.

The conservative black-box design is central. A white-box solver may be faster
when the verifier exposes stable structure, and our theorem does not deny that.
The safety motivation is different: when verifier internals are complex,
changing, hidden, or adversarially coupled, the robust assumption is to make no
structural assumption at all. Quantum minimum finding is valuable precisely
because it retains a square-root query guarantee in that worst-case oracle
model.

When is \method\ preferable? If the verifier is convex, separable, or otherwise
amenable to source-aware optimization, branch-and-bound, CP-SAT, MILP, or local
search should be used. \method\ targets the least structured regime: changing
rulebooks, proprietary shields, relabeled semantics, or adversarially coupled
predicates where no stable neighborhood model should be assumed. In that regime,
the square-root verifier-query guarantee is a robustness guarantee rather than a
near-term hardware-speedup claim.

The problem template is broader than driving. The pattern ``hard-rule veto
\(\rightarrow\) finite multi-owner repair \(\rightarrow\) auditable
certificate'' also appears in robot collaboration, air-traffic management, and
smart-factory scheduling. We evaluate driving because it exposes concrete
right-of-way and fallback semantics; the verifier-shielded search idea applies
whenever a domain supplies a finite repair lattice and a deterministic
certificate verifier.

Finally, ``quantum proposes; \carve\ certifies'' is a transferable safe
integration paradigm. A noisy, suboptimal, or even adversarial quantum module
can only propose. A deterministic classical verifier retains sole authority, so
proposal errors cannot become safety certificates. This separation addresses a
core adoption concern for quantum-AI in safety-critical systems.

The claim boundary is precise. Supported claims are certified interactive
repair semantics, verifier-shielded certificate soundness, black-box
query-complexity separation, reversible-oracle constructibility, statevector
oracle evidence, and replay certificate preservation. Excluded claims are
present-day quantum hardware speedup, universal quantum superiority, QAOA
dominance, quantum safety certification, and real-world deployment readiness.
This boundary strengthens the paper because every safety-critical property is
checked by \carve.

\section{Conclusion}

We presented \carve, a certified repair architecture that elevates a hard-rule
veto from terminal rejection into an auditable certificate, and \method, a
verifier-shielded quantum search layer for the resulting joint-repair
bottleneck. The key shift is to separate what is certified from how it is found.
\carve\ defines and verifies the binding rule, repair assignment,
responsibility split, right-of-way envelope, and fallback; \method\ searches the
black-box lattice with theorem-backed minimum finding but never becomes a safety
authority. The broader pattern is deliberately trust-bounded: quantum proposes;
\carve\ certifies.

\bibliography{references}

\end{document}